\documentclass[11pt]{article}
\usepackage[left=2.0cm,top=2.5cm,right=2.0cm,bottom=2.5cm,bindingoffset=0.0cm]{geometry}

\usepackage{graphicx} 
\usepackage{subfigure}

\usepackage{algorithm}
\usepackage{algorithmic}

\usepackage{hyperref}

\floatstyle{plain}
\newfloat{myalgo}{tbhp}{mya}

\newenvironment{Algorithm}[1][tbh]%
{\begin{myalgo}[#1]
\centering
\begin{minipage}{0.6\linewidth}
\begin{algorithm}[H]}%
{\end{algorithm}
\end{minipage}
\end{myalgo}}

\usepackage{bm}
\usepackage{amsmath}
\usepackage{amsthm}
\usepackage{amsfonts}
\usepackage{amssymb}
\usepackage{color}
\theoremstyle{plain}
\newtheorem{prop}{Proposition}
\newtheorem{lemma}{Lemma}

\theoremstyle{definition}

\newcommand{\becca}[1]{#1}

\def\ie{{\em i.e.,~}}
\def\eg{{\em e.g.,~}}
\def\cf{{\em cf.,~}}

\DeclareMathOperator*{\argmin}{arg\,min}

\begin{document} 

\title{Algebraic Variety Models for High-Rank Matrix Completion}

\author{Greg Ongie\thanks{Department of EECS, University of Michigan,
            Ann Arbor, MI, USA. (\emph{gongie@umich.edu}).}
\and
Rebecca Willett\thanks{Department of ECE, University of Wisconsin-Madison,
             Madison, WI, USA. (\emph{willett@discovery.wisc.edu}).}
\and
Robert D. Nowak\thanks{Department of ECE, University of Wisconsin-Madison,
             Madison, WI, USA. (\emph{rdnowak@wisc.edu}).}
\and
Laura Balzano\thanks{Department of EECS, University of Michigan,
            Ann Arbor, MI USA. (\emph{girasole@umich.edu}).}
}

\maketitle

\begin{abstract} 
We consider a generalization of low-rank matrix completion to the case where the data belongs to an algebraic variety, \ie each data point is a solution to a system of polynomial equations. In this case the original matrix is possibly high-rank, but it becomes low-rank after mapping each column to a higher dimensional space of monomial features. Many well-studied extensions of linear models, including affine subspaces and their union, can be described by a variety model. In addition, varieties can be used to model a richer class of nonlinear quadratic and higher degree curves and surfaces. We study the sampling requirements for matrix completion under a variety model with a focus on a union of affine subspaces. We also propose an efficient matrix completion algorithm that minimizes a convex or non-convex surrogate of the rank of the matrix of monomial features. Our algorithm uses the well-known ``kernel trick'' to avoid working directly with the high-dimensional monomial matrix. We show the proposed algorithm is able to recover synthetically generated data up to the predicted sampling complexity bounds. The proposed algorithm also outperforms standard low rank matrix completion and subspace clustering techniques in experiments with real data.
\end{abstract} 

\section{Introduction}
\label{sec:submission}
Work in the last decade on matrix completion has shown that it is possible to leverage linear structure in order to interpolate missing values in a low-rank matrix \cite{candes2012exact}. The high-level idea of this work is that if the data defining the matrix belongs to a structure having fewer degrees of freedom than the entire dataset, that structure provides redundancy that can be leveraged to complete the matrix. The assumption that the matrix is low-rank is equivalent to assuming the data lies on (or near) a low-dimensional linear subspace. 

It is of great interest to generalize matrix completion to exploit low-complexity \emph{nonlinear} structures in the data. Several avenues have been explored in the literature, from generic manifold learning \cite{lee2013local}, to unions of subspaces \cite{eriksson2012high,elhamifar2013sparse}, to low-rank matrices perturbed by a nonlinear monotonic function \cite{ganti2015matrix,song2016blind}. In each case missing data has been considered, but there lacks a clear, unifying framework for these ideas.

In this work we study the problem of completing a matrix whose columns belong to an \emph{algebraic variety}, \ie the set of solutions to a system of polynomial equations \cite{cox15}. This is a strict generalization of the linear (or affine) subspace model, which can be written as the set of points satisfying a system of linear equations. Unions of subspaces and unions of affine spaces also are algebraic varieties. In addition, a much richer class of non-linear curves, surfaces, and their unions, are captured by a variety model.

The matrix completion problem using a variety model can be formalized as follows. Let $\bm X = \left[\begin{matrix} \bm x_1 ,\dots , \bm x_s \end{matrix} \right] \in \mathbb{R}^{n\times s}$ be a matrix of $s$ data points where each column $\bm x_i \in \mathbb{R}^n$. Define $\phi_d:\mathbb{R}^n\rightarrow\mathbb{R}^N$ as the mapping that sends the vector $\bm x = (x_1,...,x_n)$ to the vector of all monomials in $x_1,...,x_n$ of degree at most $d$, and let $\phi_d(\bm X)$ denote the matrix that results after applying $\phi_d$ to each column of $\bm X$, which we call the \emph{lifted matrix}. We will show the lifted matrix is rank deficient if and only if the columns of $\bm X$ belong to an algebraic variety. This motivates the following matrix completion approach:
\begin{equation}
  \min_{\bm{\hat X}}~~\text{rank}\, \phi_d(\bm{\hat X})~~\text{such that}~~\mathcal{P}_\Omega(\bm{\hat X}) = \mathcal{P}_\Omega(\bm X)
  \label{eq:rankmin}
\end{equation}
where $\mathcal{P}_\Omega( \cdot )$ represents a projection that restricts to some observation set $\Omega \subset \{1,\dots, n\} \times \{1,\dots,s\}$. The rank of $\phi_d(\bm{\hat X})$ depends on the choice of the polynomial degree $d$ and the underlying ``complexity'' of the variety, in a sense we will make precise. Figure \ref{fig:datasets} shows two examples of datasets that have low-rank in the lifted space for different polynomial degree.

In this work we investigate the factors that influence the sampling complexity of varieties as well as algorithms for completion. The challenges are (a) to characterize varieties having low-rank (and therefore few degrees of freedom) in the lifted space, \ie determine when $\phi_d(\bm X)$ is low-rank, 
(b) devise efficient algorithms for solving \eqref{eq:rankmin} that can exploit these few degrees of freedom in a matrix completion setting, and (c) determine the trade-offs relative to existing matrix completion approaches. This work contributes considerable progress towards these goals.

\begin{figure}[ht!]
\centering
\begin{minipage}{0.75\linewidth}
\includegraphics[height=0.32\linewidth]{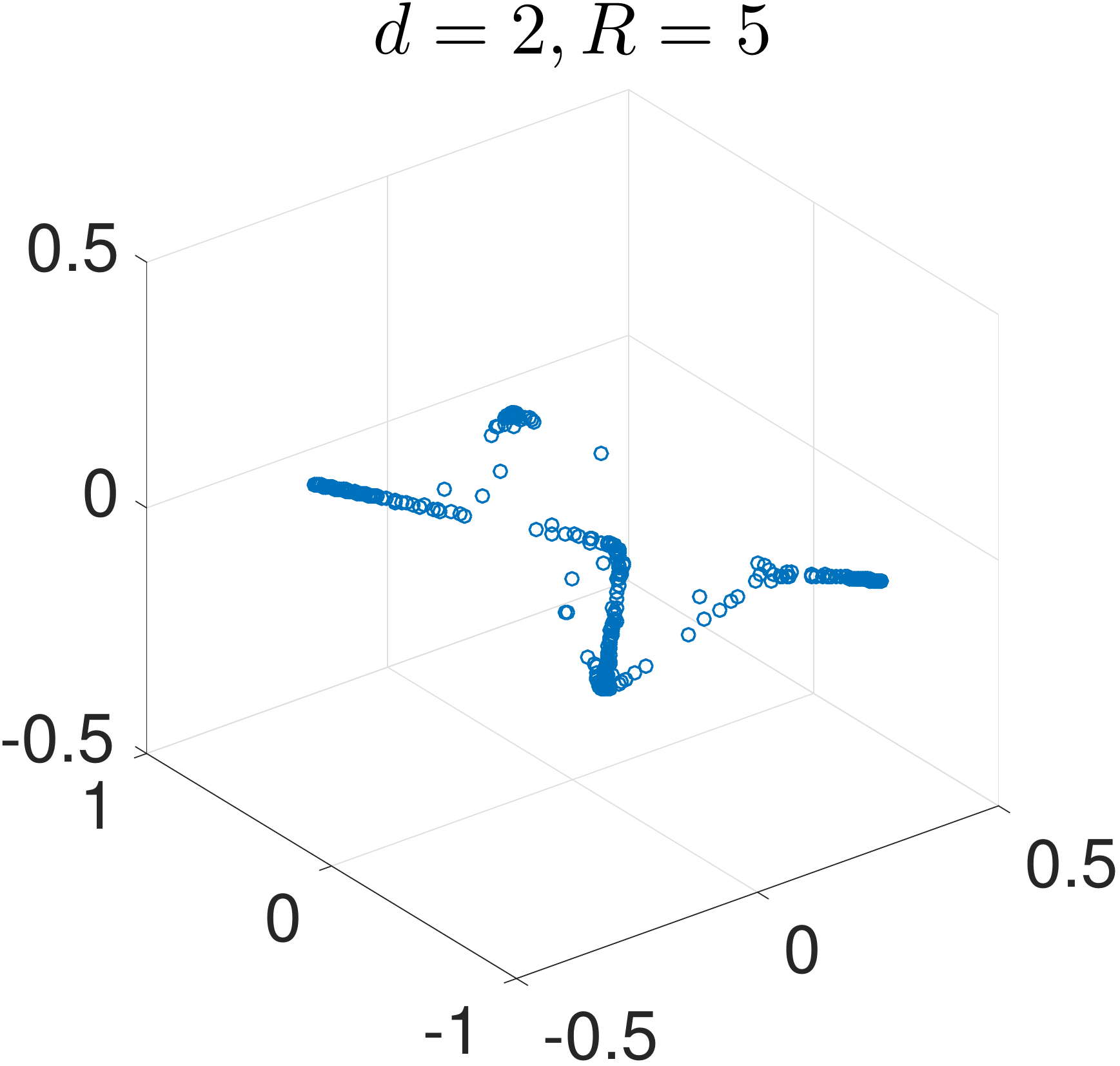}
\includegraphics[height=0.32\linewidth]{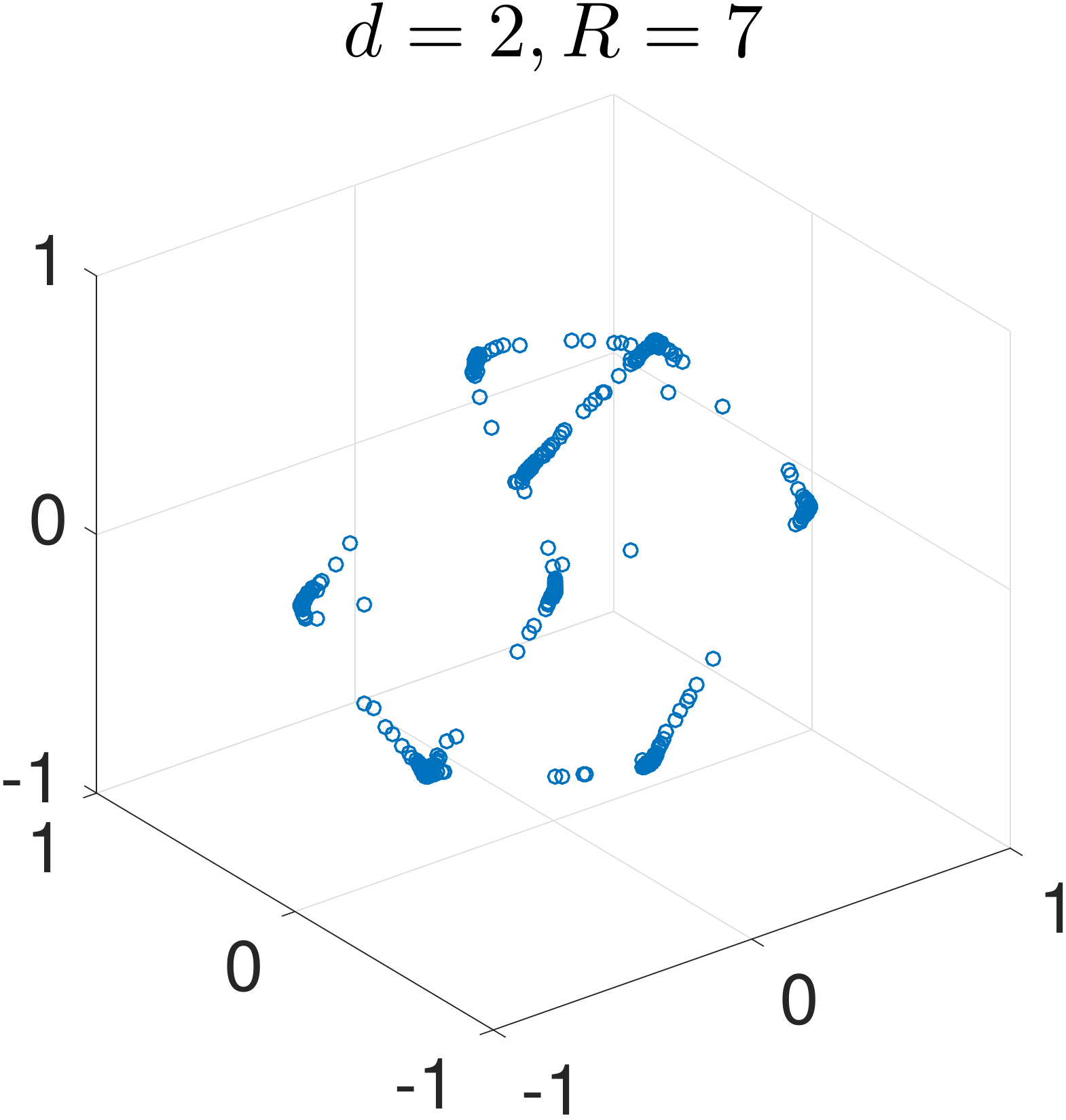}
\includegraphics[height=0.32\linewidth]{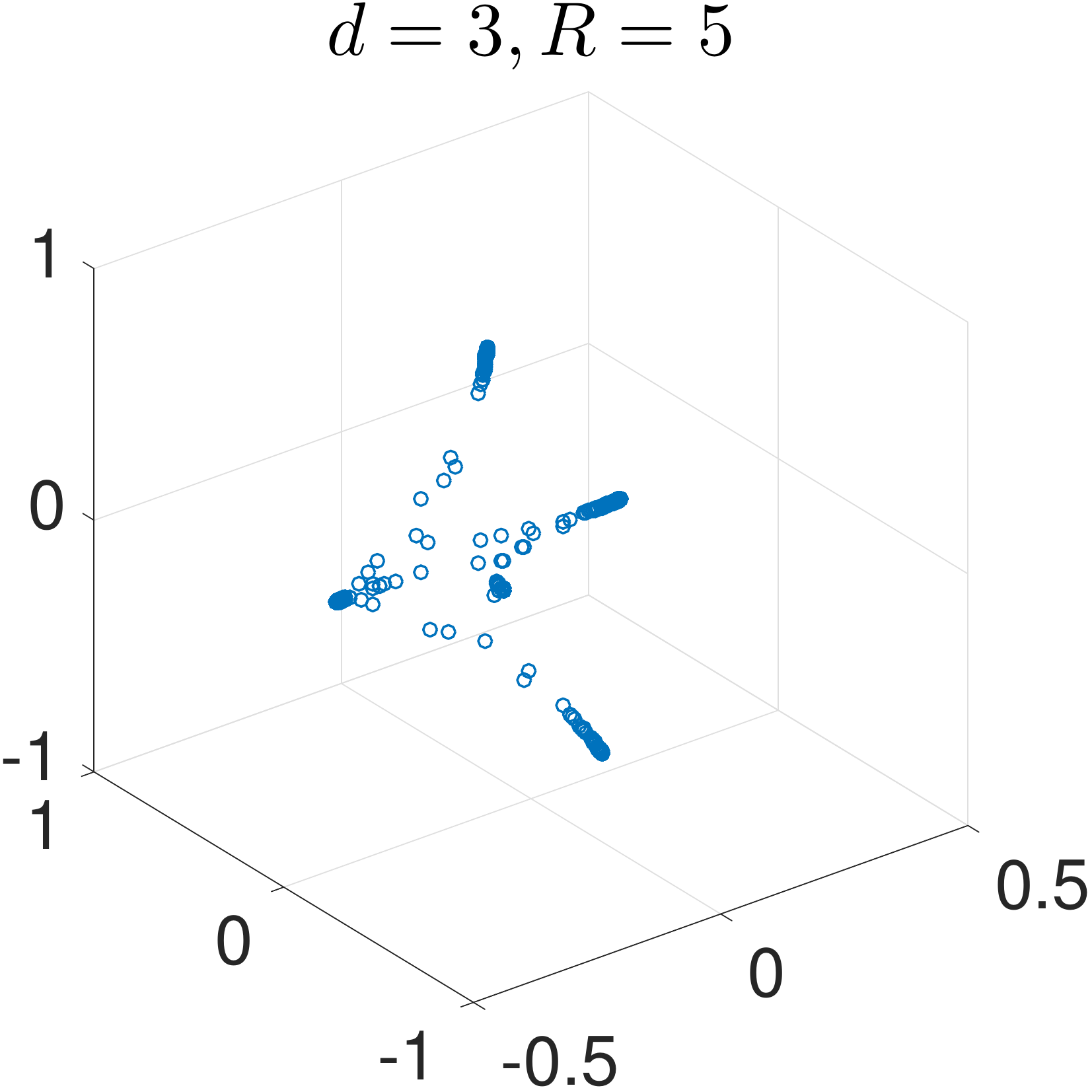}
\end{minipage}
\caption{Data belonging to algebraic varieties in $\mathbb{R}^3$. The original data is full rank, but a nonlinear embedding of the matrix to a feature space consisting of monomials of degree at most $d$ is low-rank with rank $R$, indicating the data has few degrees of freedom.}
\label{fig:datasets}
\end{figure}

For a given variety model, we seek to describe the degrees of freedom that determine the sampling complexity of the model. For example, it is well-known that $n\times n$ rank $r$ matrix can be completed from $O(r n\, \text{polylog}\,n )$ sampled entries under standard incoherence assumptions \cite{recht2011simpler}. This is very close to the $O(rn)$ degrees of freedom for such a matrix. 
Similarly, the degrees of freedom in the lifted space is $O(RN)$, where $R$ is the rank of the lifted matrix and $N$ is the number of higher-degree monomials. This is suggestive of the number of samples required for completion in the lifted space and in turn the number of samples required in the original observation space.
We note that although $N>n$ and $R>r$, for many varieties $r/n \gg R/N$, implying potential for completion in the lifted space.

Our contributions are as follows. We identify bounds on the rank of a matrix $\phi_d(\bm X)$ when the columns of the data matrix $\bm X$ belong to an algebraic variety. We study how many entries of such a matrix should be observed in order to recover the full matrix from an incomplete sample. We show as a case study that monomial representations produce low-rank representations of unions of subspaces, and we characterize the rank. The standard union of subspace representation as a discrete collection of individual subspaces is inherently non-smooth in nature, whereas the algebraic variety allows for a purely continuous parameterization. This leads to general algorithms for completion of a data matrix whose columns belong to a variety. The algorithms' performance are showcased on data simulated as a union of subspaces, a union of low-dimensional parametric surfaces, and real data from a motion segmentation dataset and a motion capture dataset. The simulations show that the performance of our algorithm matches our predictions and outperforms other methods. 
In addition, the analysis of the degrees of freedom associated with the proposed representations introduces several new research avenues at the intersection of nonlinear algebraic geometry and random matrix theory.

\subsection{Related Work}
There has been a great deal of research activity on matrix completion problems since \cite{candes2012exact}, where the authors showed that one can recover an incomplete matrix from few entries using 
a convex relaxation of the rank minimization optimization problem. At this point it is even well-known that $O(rn)$ entries are necessary and sufficient \cite{pimentel2016characterization} for almost every matrix as long as 
the measurement pattern satisfies certain deterministic conditions. However, these methods and theory are restricted to low-rank linear models. A great deal of real data exhibit nonlinear structure, and so it is of interest to generalize this approach.

Work in that direction has dealt with union of subspaces models \cite{eriksson2012high,yang2015sparse,elhamifar2016high,pimentel2016group,pimentel2016information}, locally linear approximations \cite{lee2013local}, 
as well as low-rank models perturbed by an arbitrary nonlinear link function \cite{ganti2015matrix,song2016blind,rao2017learning}. 
In this paper we instead seek a more general model that captures both linear and nonlinear structure. 
The variety model has as instances low-rank subspaces and their union as well as quadratic and higher degree curves and surfaces. 

\becca{Work on kernel PCA (\cf \cite{sanguinetti2006missing,nguyen2009robust}) leverage similar geometry to ours. In {\em Kernel Spectral Curvature Clustering} \cite{chen2009kernel}, the authors similarly consider clustering of data points via subspace clustering in a lifted space using kernels. These works are algorithmic in nature, with promising numerical experiments, but do not systematically consider missing data or analyze relative degrees of freedom.}

This paper also has close ties to algebraic subspace clustering (ASC) \cite{vidal2003algebraic,vidal2005generalized,vidal16,tsakiris2015algebraic}, also known as generalized PCA. Similar to our approach, the ASC framework models unions of subspaces as an algebraic variety, and makes use of monomial liftings of the data to identify the subspaces. Characterizations of the rank of data belonging to union of subspaces under the monomial lifting are used in the ASC framework \cite{vidal16} based on results in \cite{derksen2007hilbert}. The difference of the results in \cite{derksen2007hilbert} and those in Prop.~\ref{prop:rankbound_onesub} is that ours hold for monomial liftings of all degrees $d$, not just $d \geq k$, where $k$ is the number of subspaces. Also, the main focus of ASC is to recover unions of subspaces or unions of affine spaces, whereas we consider data belonging to a more general class of algebraic varieties. Finally, the ASC framework has not been adapted to the case of missing data, which is the main focus of this work.

\section{Variety Models}\label{sec:model}
\subsection{Toy example}
As a simple example to illustrate our approach, consider a matrix
\[
\bm X = 
\begin{pmatrix}
	x_{1,1} & x_{1,2} & \cdots & x_{1,6}\\
	x_{2,1} & x_{2,2} & \cdots & x_{2,6}\\
\end{pmatrix}\in\mathbb{R}^{2\times 6}
\]
whose six columns satisfy the quadratic equation
\begin{equation}\label{eq:unitcircle}
c_0 + c_1\,x_{1,i} + c_2\,x_{2,i} + c_3\,x_{1,i}^2 + c_4\,x_{1,i} x_{2,i} + c_5\,x_{2,i}^2= 0
\end{equation}
for $i = 1,\ldots,6$ and some unknown constants $c_0,...,c_5$ that are not all zero. Generically, $\bm X$ will be full rank. However, suppose we vertically expand each column of the matrix to make a $6\times 6$ matrix
\[
\bm Y = 
\begin{pmatrix}
	1   & 1 & \cdots & 1\\
	x_{1,1} & x_{1,2} & \cdots & x_{1,6}\\
	x_{2,1} & x_{2,2} & \cdots & x_{2,6}\\
    x_{1,1}^2 & x_{1,2}^2 & \cdots & x_{1,6}^2\\
    x_{1,1}x_{2,1} & x_{1,2}x_{2,2} & \cdots & x_{1,6}x_{2,6}\\
	x_{2,1}^2 & x_{2,2}^2 & \cdots & x_{2,6}^2\\
\end{pmatrix},
\]
\ie we augment each column of $\bm X$ with a $1$ and with the quadratic monomials $x_{1,i}^2$, $x_{1,i}x_{2,i}$, $x_{2,i}^2$. This allows us to re-express the polynomial equation \eqref{eq:unitcircle} as the matrix-vector product
\[
\bm Y^T \bm c = \bm{0}
\]
where $\bm c = (c_0,c_1,..,c_5)^T$. In other words, $\bm Y$ is rank deficient. Suppose, for example, that we are missing entry $x_{1,1}$ of $\bm X$. Since $\bm X$ is full rank, there is no way to uniquely complete the missing entry by leveraging linear structure alone. Instead, we ask: Can we complete $x_{1,1}$ using the linear structure present in $\bm Y$? Due to the missing entry $x_{1,1}$, the first column of $\bm Y$ will having the following pattern of missing entries: $(1,-,x_{2,1},-,-,x_{2,1}^2)^T$. However, assuming the five complete columns in $\bm Y$ are linearly independent, we can uniquely determine the nullspace vector $\bm c$ up to a scalar multiple. Then from \eqref{eq:unitcircle} we have
\[
c_3\, x_{1,1}^2 + (c_1 + c_4\, x_{2,1}) x_{1,1} = -c_0 -c_2\,x_{2,1}-c_5\,x_{2,1}^2
\]
In general, this equation will yield at most two possibilities for $x_{1,1}$. Moreover, there are conditions where we can \emph{uniquely} recover $x_{1,1}$, namely when $c_3 = 0$ and $c_1 + c_4\, x_{2,1} \neq 0$.

This example shows that even without a priori knowledge of the particular polynomial equation satisfied by the data, it is possible to uniquely recover missing entries in the original matrix by leveraging induced linear structure in the matrix of expanded monomials. We now show how to considerably generalize this example to the case of data belonging to an arbitrary algebraic variety.

\subsection{Formulation}
Let $\bm X = \left[\begin{matrix} \bm x_1 ,\dots , \bm x_s \end{matrix} \right] \in \mathbb{R}^{n\times s}$ be a matrix of $s$ data points where each column $\bm x_i \in \mathbb{R}^n$. Define $\phi_d:\mathbb{R}^n\rightarrow\mathbb{R}^N$ as the mapping that sends the vector $\bm x = (x_1,...,x_n)$ to the vector of all monomials in $x_1,...,x_n$ of degree at most $d$:
\begin{equation}\label{eq:featuremap}
	\phi_d(\bm x) = ({\bm x}^{\bm \alpha})_{|\bm \alpha| \leq d} \in \mathbb{R}^N
\end{equation}
where $\bm \alpha = (\alpha_1,...,\alpha_n)$ is a multi-index of non-negative integers, with $\bm x^{\bm \alpha} := x_1^{\alpha_1}\cdots x_n^{\alpha_n}$, and $|\bm \alpha| := \alpha_1+\cdots+\alpha_n$. In the context of kernel methods in machine learning, the map $\phi_d$ is often called a polynomial feature map \cite{muller2001introduction}. Borrowing this terminology, we call $\phi_d(\bm x)$ a \emph{feature vector}, the entries of $\phi_d(\bm x)$ \emph{features}, and the range of $\phi_d$ \emph{feature space}. Note that the number of features is given by $N = N(n,d)= \binom{n+d}{n} = \binom{n+d}{d}$, the number of unique monomials in $n$ variables of degree at most $d$. When $\bm X = [\bm x_1,...,\bm x_s]$ is an $n\times s$ matrix, we use $\phi_d(\bm X)$ to denote the $N\times s$ matrix $[\phi_d(\bm x_1),...,\phi_d(\bm x_s)]$. 

The problem we consider is this: can we complete a partially observed matrix $\bm X$ under the assumption that $\phi_d(\bm X)$ is low-rank? This can be posed as the optimization problem given above in Equation \eqref{eq:rankmin}. 
We give a practical algorithm for solving a relaxation of \eqref{eq:rankmin} in Section \ref{sec:alg}. \becca{Similar to previous work cited above on using polynomial feature maps, our method leverages the {\em kernel trick} for efficient computations. However, it would be na\"ive to think of the associated analysis as applying known results on matrix completion sample complexities to our high-dimensional feature space. In particular, if we observe $m$ entries per column in a rank-$r$ matrix of size $n \times s$ and apply the polynomial feature map, then in the feature space we have $M=\binom{m+d}{d}$ entries per column in a rank-$R$ matrix of size $N\times s$. Generally, the number of samples, rank, and dimensional all grow in the mapping to feature space, but they grow at different rates depending on the underlying geometry; it is not immediately obvious what conditions on the geometry and sampling rates impact our ability to determine the missing entries.}
In the remainder of this section, we show how to relate the rank of $\phi_d(\bm X)$ to the underlying variety, and we study the sampling requirements necessary for the completion of the matrix in feature space.

\subsection{Rank properties}\label{subsec:rankprops}
To better understand the rank of the matrix $\phi_d(\bm X)$, we introduce some additional notation and concepts from algebraic geometry. Let $\mathbb{R}[\bm x]$ denote the space of all polynomials with real coefficients in $n$ variables $\bm x = (x_1,...,x_n)$. We model a collection of data as belonging to a \emph{real (affine) algebraic variety} \cite{cox15}, which is defined as the common zero set of a system of polynomials $P \subset \mathbb{R}[\bm x]$:
\[
V(P) = \{\bm x \in\mathbb{R}^n: f(\bm x) = 0\text{ for all }f\in P\}.
\]
Suppose the variety $V(P)$ is defined by the finite set of polynomials $P = \{f_1,...,f_q\}$, where each $f_i$ has degree at most $d$. Let $\bm C \in \mathbb{R}^{N\times q}$ be the matrix whose columns are given by the vectorized coefficients $(c_{\bm \alpha,i})_{|\bm \alpha|\leq d}$ of the polynomials $f_i(x),~i = 1,...,q$ in $P$. Then the columns of $\bm X$ belong to the variety $V(P)$ if and only if $\phi_d(\bm X)^T \bm C = \bm{0}$. In particular, assuming the columns of $\bm C$ are linearly independent, this shows that $\phi_d(\bm X)$ has rank $\leq \min(N-q,s)$. In particular, when the number of data points $s > N-q$, then $\phi_d(\bm X)$ is rank deficient.

However, the exact rank of $\phi_d(\bm X)$ could be much smaller than $\min(N-q,s)$, especially when the degree $d$ is large. This is because the coefficients $\bm c$ of \emph{any polynomial} that vanishes at every column of $\bm X$ satisfies $\phi_d(\bm X)^T\bm c = \bm{0}$. We will find it useful to identify this space of coefficients with a finite dimensional vector space of polynomials. 
 Let $\mathbb{R}_d[\bm x]$ be the space of all polynomials in $n$ real variables of degree at most $d$. We define \emph{vanishing ideal of degree $d$} corresponding to a set $\mathcal{X}\subset\mathbb{R}^n$, denoted by $\mathcal{I}_d(\mathcal{X})$, to be subspace of polynomials belonging to $\mathbb{R}_d[\bm x]$ that vanish at all points in $\mathcal{X}$:
\begin{equation}
\mathcal{I}_d(\mathcal{X}) := \{f\in \mathbb{R}_d[\bm x]: f(\bm x)=0~\text{for all}~\bm x \in \mathcal{X}\}.
\end{equation}
We also define the \emph{non-vanishing ideal of degree $d$} corresponding to $X$, denoted by $\mathcal{S}_d(\mathcal{X})$, to be the orthogonal complement of $\mathcal{I}_d(\mathcal{X})$ in $\mathbb{R}_d[\bm x]$:
\begin{equation}
\mathcal{S}_d(\mathcal{X}) := \{ g \in \mathbb{R}_d[\bm x] : \langle f,g\rangle = 0~~\text{for all}~~f\in \mathcal{I}_d(\mathcal{X})\},
\end{equation}
where the inner product $\langle f,g\rangle$ of polynomials $f,g\in\mathbb{R}_d[\bm x]$ is defined as the inner product of their coefficient vectors. 
Hence, the rank of a data matrix in feature space can be expressed in terms of the dimension of non-vanishing ideal of degree $d$ corresponding to $\mathcal{X} = \{\bm x_1,....,\bm x_s\}$, the set of all columns of $\bm X$. Specifically, we have $\text{rank}~\phi_d(\bm X) = \min(R,s)$ where
\begin{equation}\label{eq:rankbound}
R = \text{dim}~\mathcal{S}_d(\mathcal{X}) = N-\text{dim}~\mathcal{I}_d(\mathcal{X})\;.
\end{equation}
This follows from the rank-nullity theorem, since $\mathbb{R}_d[\bm x]$ has dimension $N$.
In general the dimension of the space $\mathcal{I}_d(\mathcal{X})$ or $\mathcal{S}_d(\mathcal{X})$ is difficult to determine when $\mathcal{X}$ is an arbitrary set of points. However, if we assume $\mathcal{X}$ is a subset of a variety $V$, since $\mathcal{I}_d(V) \subseteq \mathcal{I}_d(\mathcal{X})$ we immediately have the bound
\begin{equation}\label{eq:rankbound_var}
\text{rank}~\phi_d(\bm X) \leq \text{dim}~\mathcal{S}_d(V).
\end{equation}
In certain cases $\text{dim}~\mathcal{S}_d(V)$ can be computed exactly or bounded using properties of the polynomials defining $V$. For example, it is possible to compute the dimension of $\mathcal{S}_d(V)$ directly from a \emph{Gr\"obner basis} for the vanishing ideal associated with $V$ \cite{cox15}.
In Section \ref{sec:uos} we show how to bound the dimension of $\mathcal{S}_d(V)$ in the case where $V$ is a union of subspaces.

\subsection{Sampling rate}\label{subsec:sampling}
Informally, the \emph{degrees of freedom} of a class of objects is the minimum number of free variables needed to describe an element in that class uniquely. For example, a $n\times s$ rank $r$ matrix has $r(n+s-r)$ degrees of freedom: $nr$ parameters to describe $r$ linearly independent columns making up a basis of the column space, and $r(s-r)$ parameters to describe the remaining $s-r$ columns in terms of this basis. It is impossible to uniquely complete a matrix in this class if we sample fewer than this many entries.

We can make a similar argument to specify the minimum number of samples needed to uniquely complete a matrix that is low-rank when mapped to feature space. 
First, we characterize how missing entries of the data matrix translate to missing entries in feature space. For simplicity, we will assume a sampling model where we sample a fixed number of entries $m$ from each column of the original data matrix. Let $\bm x \in \mathbb{R}^n$ represent a single column of the data matrix, and $\Omega \subset \{1,...,n\}$ with $m = |\Omega|$ denote the indices of the sampled entries of $\bm x$. The pattern of revealed entries in $\phi_d(\bm x)$ corresponds to the set of multi-indices: 
\[
\{\bm \alpha = (\alpha_1,...,\alpha_n) : |\bm \alpha| \leq d,~\alpha_i = 0 \text{ for all } i \in \Omega^c\},
\]
which has the same cardinality as the set of all monomials of degree at most $d$ in $m$ variables, \ie $\binom{m+d}{d}$. If we call this quantity $M$, then the ratio of revealed entries in $\phi_d(x)$ to the feature space dimension is
\[
\frac{M}{N} = \frac{\binom{m+d}{d}}{\binom{n+d}{d}} = \frac{(m+d)(m+d-1)\cdots(m+1)}{(n+d)(n+d-1)\cdots(n+1)},
\]
which is on the order of $(\frac{m}{n})^d$ for small $d$. More precisely, we have the bounds
\begin{equation}
  \left(\frac{m}{n}\right)^d \leq \frac{M}{N} \leq \left(\frac{m+d}{n}\right)^d,
\end{equation}
and consequently
\begin{equation}\label{eq:orderMN}
    \frac{m}{n} \leq \left(\frac{M}{N}\right)^{\frac{1}{d}} \leq \frac{m}{n}+\frac{d}{n}.
\end{equation}

In total, observing $m$ entries per column of the data matrix translates to $M$ entries per column in feature space. Suppose the $N\times s$ lifted matrix $\phi_d(\bm X)$ is rank $R$. By the preceding discussion, we need least $R(N+s-R)$ entries of the feature space matrix $\phi_d(\bm X)$ to complete it uniquely among the class of all $N\times s$ matrices of rank $R$. Hence, at minimum we need to satisfy
\begin{equation}\label{eq:minsamp}
Ms \geq R{(N+s-R)}.
\end{equation}
Let $m_0$ denote the minimal value of $m$ such that $M=\binom{m+d}{d}$ achieves the bound \eqref{eq:minsamp}, and set $M_0=\binom{m_0+d}{d}$. Dividing \eqref{eq:minsamp} through by the feature space dimension $N$ and $s$ gives
\begin{equation}\label{eq:boundMN}
\frac{M_0}{N} \geq \left(\frac{R}{N}\right)\left(\frac{N+s-R}{s}\right) = \left(\frac{R}{s} + \frac{R}{N}\left(1-\frac{R}{s}\right)\right),
\end{equation}
and so
from \eqref{eq:orderMN} we see we can guarantee this bound by having 
\begin{equation}\label{eq:sampbound}
\rho_0:= \frac{m_0}{n} \geq \left(\frac{R}{s} + \frac{R}{N}\left(1-\frac{R}{s}\right)\right)^{\frac 1 d}\;,
\end{equation} and this in fact will result in tight satisfaction of \eqref{eq:boundMN} because $\left(M_0/N\right)^{\frac 1 d} \approx m_0/n$ for small $d$ and large $n$. 

At one extreme where the matrix $\phi_d(\bm X)$ is full rank, then $R/s = 1$  or $R/N=1$ and according to \eqref{eq:sampbound} we need $\rho_0 \approx 1$, \ie full sampling of every data column.
At the other extreme where instead we have many more data points than the feature space rank, $R/s \ll 1$, then \eqref{eq:sampbound} gives the asymptotic bound $\rho_0 \approx \left(R/N\right)^{\frac{1}{d}}$.

The above discussion bounds the degrees of freedom of a matrix that is rank-$R$ in feature space. Of course, the proposed variety model has potentially fewer degrees of freedom than this, because additionally the columns of the lifted matrix are constrained to lie in the image of the feature map. We use the above bound only as a rule of thumb for sampling requirements on our matrix. Furthermore, we note that sample complexities for standard matrix completion often require that locations are observed uniformly at random, whereas in our problem the locations of observations in the lifted space will necessarily be structured. However, there is recent work that shows matrix completion can suceed without these assumptions \cite{pimentel2016characterization, chen2014coherent} that gives reason to believe random samples in the original space may allow completion in the lifted space, and our empirical results support this rationale.

\section{Case Study: Union of affine subspaces}\label{sec:uos}
A union of affine subspaces can also be modeled as an algebraic variety. For example, with $(x,y,z)\in\mathbb{R}^3$, the union of the plane $z = 1$ and the line $x = y$ is the zero-set of the \emph{quadratic} polynomial $q(x,y,z) = (z-1)(x-y)$. In general, if $\mathcal{A}_1,\mathcal{A}_2 \subset \mathbb{R}^n$ are affine spaces of dimension $r_1$ and $r_2$, respectively, then we can write $\mathcal{A}_1 = \{\bm x \in \mathbb{R}^n : f_i(\bm x) = 0~\text{for}~i=1,...,n-r_1\}$ and $\mathcal{A}_2 = \{\bm x : g_i(\bm x) = 0~\text{for}~i=1,...,n-r_2\}$ where the $f_i$ and $g_i$ are linear, and their union $\mathcal{A}\cup\mathcal{B}$ can be expressed as the common zero set of all possible products of the $f_i$ and $g_i$:
\begin{equation}
\mathcal{A}_1\cup\mathcal{A}_2 = \{\bm x \in \mathbb{R}^n : f_i(\bm x)g_j(\bm x) = 0~\text{for all}~1\leq i\leq n-r_1,~1\leq j\leq n-r_2\}.
\end{equation}
\ie $\mathcal{A}_1\cup\mathcal{A}_2$ is the common zero set of a system of $(n-r_1)(n-r_2)$ quadratic equations. This argument can be extended to show a union of $k$ affine subspaces of dimensions $r_1,...,r_k$ is a variety described by a system of $\prod_{i=1}^k(n-r_i)$ polynomial equations of degree $k$.

In this section we establish bounds on the feature space rank for data belonging to a union of subspaces. We will make use of the following lemma that shows the dimension of a vanishing ideal is fixed under an affine change of variables:
\begin{lemma}\label{prop:affinetrans}
Let $\mathcal{T}:\mathbb{R}^n\rightarrow\mathbb{R}^n$ be an affine change of variables, \ie $\mathcal{T}(\bm x) = \bm A \bm x + \bm b$, where $\bm b \in \mathbb{R}^n$ and $\bm A \in \mathbb{R}^{n\times n}$ is invertible. Then for any $S\subset \mathbb{R}^n$,
\begin{equation}
\dim \mathcal{I}_d(S) = \dim \mathcal{I}_d(\mathcal T(S)).
\end{equation}
\end{lemma}
We omit the proof for brevity, but the result is elementary and relies on the fact the degree of a polynomial is unchanged under an affine change of variables. Our next result establishes a bound on the feature space rank for a single affine subspace:

\begin{prop}\label{prop:rankbound_onesub}
Let the columns of a matrix $\bm X^{n\times s}$ belong to an affine subspace $\mathcal{A}\subset \mathbb{R}^n$ of dimension $r$. Then
\begin{equation}\label{eq:rankbound_onesub}
\text{rank}\,\phi_d(\bm X) \leq \binom{r+d}{d},~~\text{for all}~~d\geq 1.
\end{equation}
\end{prop}
\begin{proof}
By Lemma~\ref{prop:affinetrans}, $\text{dim}\, \mathcal{I}_d(\mathcal A)$ is preserved under an affine transformation of $\mathcal{A}$. Note that we can always find an affine change of variables $\bm y = \bm A\bm x+\bm c$ with invertible $\bm A \in\mathbb{R}^{n\times n}$ and $\bm c\in\mathbb{R}^n$ such that in the coordinates $\bm y = (y_1,...,y_n)$ the variety $\mathcal{A}$ becomes
\begin{equation}
\mathcal{A} = \{(y_1,\ldots,y_r,0,\ldots,0):y_1,...,y_r\in \mathbb{R}\}.
\end{equation} 
For any polynomial $f(\bm y) = \sum_{|\bm \alpha|\leq d} c_{\bm \alpha} \bm y^{\bm \alpha} $, the only monomial terms in $f(\bm y)$ that do not vanish on $\mathcal{A}$ are those having the form $y_1^{\alpha_1}\cdots y_r^{\alpha_r}$. Furthermore, any polynomial in just these monomials that vanishes on all of $\mathcal{A}$ must be the zero polynomial, since the $y_1,...,y_r$ are free variables. Therefore,
\begin{align}
\mathcal{S}_d(\mathcal{A}) & = \text{span}\{y_1^{\alpha_1}\cdots y_r^{\alpha_r} : \alpha_1+\cdots + \alpha_r \leq d\}
\end{align}
\ie the non-vanishing ideal coincides with the space of polynomials in $r$ variables of degree at most $d$, which is $\binom{r+d}{d}$, proving the claim.
\end{proof}

We note that for $s$ sufficiently large, the bound in \eqref{eq:rankbound_onesub} becomes an equality, provided the data points are distributed generically within the affine subspace, meaning they are not the solution to additional non-trivial polynomial equations of degree at most $d$.\footnote{This is a consequence of the Hilbert basis theorem \cite{cox15}, which shows that every vanishing ideal has a finite generating set. For related discussion see Appendix C of \cite{vidal16}.}

We now derive bounds on the dimension of the non-vanishing/vanishing ideals for a union of affine varieties. Below we give a more general argument for a union of arbitrary varieties, then specialize to affine spaces.

Let $\mathcal{A},\mathcal{B}\subset\mathbb{R}^n$ be any two varieties. 
 It follows directly from definitions that
\begin{equation}
\mathcal{I}_d(\mathcal{A}\cup\mathcal{B}) = \mathcal{I}_d(\mathcal{A})\cap\mathcal{I}_d(\mathcal{B}).
\end{equation}
Applying orthogonal complements to both sides above gives
\begin{equation}
\mathcal{S}_d(\mathcal{A}\cup\mathcal{B}) = (\mathcal{I}_d(\mathcal{A}\cup\mathcal{B}))^\perp = \mathcal{S}_d(\mathcal{A})+\mathcal{S}_d(\mathcal{B}).
\end{equation}
Therefore, we have the bound
\begin{equation}\label{eq:unionrankboundtwo}
\text{dim}\,\mathcal{S}_d(\mathcal{A}\cup\mathcal{B}) \leq \text{dim}\,\mathcal{S}_d(\mathcal{A})+\text{dim}\,\mathcal{S}_d(\mathcal{B}).
\end{equation}
In the case of an arbitrary union of varieties $\cup_{i=1}^k \mathcal{A}_i$, repeated application of \eqref{eq:unionrankboundtwo} gives
\begin{equation}\label{eq:unionrankboundarb}
\text{dim}\,\mathcal{S}_d\left(\bigcup_{i=1}^k\mathcal{A}_i\right) \leq \sum_{i=1}^k\text{dim}\,\mathcal{S}_d(\mathcal{A}_i).
\end{equation}
Specializing to the case where each $\mathcal{A}_i$ is affine subspace of dimension at most $r$ and applying the result in Prop.~\ref{prop:rankbound_onesub} gives the following result:
\begin{prop}\label{prop:boundunion}
Let the columns of a matrix $\bm X^{n\times s}$ belong to a union of $k$ affine subspaces $\mathcal{A}_1$,..., $\mathcal{A}_k$ each having dimension at most $r$. Then we have the bound: 
\begin{equation}\label{eq:prop:boundunion}
\text{rank}\,\phi_d(\bm X) \leq  k \binom{r+d}{d},~~\text{for all}~d\geq 1.
\end{equation}
\end{prop}
We remark that in some cases the bound \eqref{eq:prop:boundunion} is (nearly) tight. For example, if the data lies on the union of two $r$-dimensional affine subspaces $\mathcal{A}$ and $\mathcal{B}$ that are mutually orthogonal, one can show $\text{rank}\,\phi_d(\bm X) = 2\binom{r+d}{d}-1$. The rank is one less than the bound in \eqref{eq:prop:boundunion} because $\mathcal{S}_d(\mathcal{A})\cap\mathcal{S}_d(\mathcal{B})$ has dimension one, coinciding with the space of constant polynomials. Determining the exact rank for data belonging to an arbitrary finite union of subspaces appears to be a intricate problem; see \cite{derksen2007hilbert} which studies the related problem of determining the Hilbert series of a union of subspaces. Empirically, we observe that the bound in \eqref{eq:prop:boundunion} is order-optimal with respect to $k,r,$ and $d$.

The feature space rank to dimension ratio $R/N$ in this case is given by
\begin{equation}
\frac{R}{N} \approx k \frac{\binom{r+d}{d}}{\binom{n+d}{d}} \approx k\left(\frac{r}{n}\right)^d
\end{equation}
Recall that the minimum sampling rate is approximately $(R/N)^{\frac{1}{d}}$ for $s\gg R$. Hence we would need
\begin{equation}\label{eq:uossamprate}
m \approx k^{\frac{1}{d}} r.
\end{equation}
This rate is favorable to low-rank matrix completion approaches, which need $m \approx k r$ measurements per column for a union of $k$ subspaces having dimension $r$. While this bound suggests it is always better to take the degree $d$ as large as possible, this is only true for sufficiently large $s$. To take advantage of the improved sampling rate implied by \eqref{eq:uossamprate}, according to \eqref{eq:sampbound} we need the number of data vectors per subspace to be $O(r^d)$. In other words, our model is able to accommodate more subspaces with larger $d$ but at the expense of requiring exponentially more data points per subspace. We note that if the number of data points is not an issue, we could take $d=\log k$ and require only $m \approx O(r)$ observed entries per column. In this case, for moderately sized $k$ (\eg $k \leq 20$) we should choose we have $d=2$ or $3$. In fact, we find that for these values of $d$ we get excellent empirical results, as we show in Section \ref{sec:exp}. 

\section{Algorithms}\label{sec:alg}
There are several existing matrix completion algorithms that could potentially be adapted to solve a relaxation of the rank minimization problem \eqref{eq:rankmin}, such as singular value thresholding \cite{cai2010singular}, or alternating minimization \cite{jain2013low}. However, these approaches do not easily lend themselves to ``kernelized'' implementations, \ie ones that do not require forming the high-dimensional lifted matrix $\phi_d(X)$ explicitly, but instead make use of the efficiently computable \emph{kernel function} for polynomial feature maps \footnote{Strictly speaking, $k_d$ is not kernel associated with the polynomial feature map $\phi_d$ as defined in \eqref{eq:featuremap}. Instead, it is the kernel of the related map $\tilde{\phi}_d(\bm x) := \{\sqrt{c_{\bm \alpha}} \bm x^{\bm \alpha} : |\bm \alpha| \leq d\}$ where $c_{\bm\alpha}$ are appropriately chosen multinomial coefficients.}
\begin{equation}\label{eq:kernel}
k_d(\bm x,\bm y) := \langle \phi_d(\bm x),\phi_d(\bm y)\rangle =  (\bm x^T\bm y + 1)^d.
\end{equation}
For matrices $\bm X = [\bm x_1,...,\bm x_s], \bm Y = [\bm y_1,...,\bm y_s] \in \mathbb{R}^{n\times s}$, we use $k_d(\bm X,\bm Y)$ to denote the matrix whose $(i,j)$-th entry is $k_d(\bm x_i,\bm y_j)$, equivalently,
\[
k_d(\bm X,\bm Y) = (\bm X^T\bm Y + \bm 1)^{\odot d}.
\]
where $\bm 1 \in \mathbb{R}^{s\times s}$ is the matrix of all ones, and $(\cdot)^{\odot d}$ denotes the entrywise $d$-th power of a matrix. A kernelized implementation of the matrix completion algorithm is critical for large $d$, since the rows of the lifted matrix $N$ scales exponentially with $d$.

One class of algorithm that kernelizes very naturally is the iterative reweighted least squares (IRLS) approach of \cite{fornasier11lmr,mohan2012iterative} for low-rank matrix completion. The algorithm also has the advantage of being able to accommodate the non-convex Schatten-$p$ relaxation of the rank penalty, in addition to the convex nuclear norm relaxation. Specifically, we use an IRLS approach to solve 
\begin{equation}\label{eq:Schattenp}
\tag{VMC}
\min_{\bm X} \|\phi_d(\bm X)\|_{\mathcal{S}_p}^p~~\text{s.t.}~~\mathcal{P}_{\Omega}(\bm X) = \mathcal{P}_\Omega(\bm X_0),
\end{equation}
where $\|\bm Y\|_{\mathcal{S}_p}$ is the Schatten-$p$ quasi-norm defined as
\begin{equation}\label{eq:Sp}
\|\bm Y\|_{\mathcal{S}_p} := \left(\sum_i \sigma_i(\bm Y)^p\right)^{\frac{1}{p}},~~0 < p \leq 1
\end{equation}
with $\sigma_i(\bm Y)$ denoting the $i^{\rm th}$ singular value of $\bm Y$. Note that for $p=1$ we recover the nuclear norm. We call this optimization formulation variety-based matrix completion (VMC).

The basic idea behind the IRLS approach can be illustrated in the case of the nuclear norm. First, we can re-express the nuclear norm as a weighted Frobenius norm:
\[
\|\bm Y\|_* = \mathsf{tr}[(\bm Y^T\bm Y)^{\frac{1}{2}}] = \mathsf{tr}[(\bm Y^T\bm Y)\underbrace{(\bm Y^T\bm Y)^{-\frac{1}{2}}}_{\bm W}]
\]
and then attempt to minimize the nuclear norm of a matrix $\bm Y$ belonging to a constraint set $\mathcal{C}$ by performing the iterations
\begin{align*}
\bm W_n & = (\bm Y_n^T \bm Y_n)^{-\frac{1}{2}}\\
\bm Y_{n+1} & = \argmin_{\bm Y\in \mathcal{C}}~\mathsf{tr}[(\bm Y^T\bm Y)\bm W_n].
\end{align*}
Note the $\bm Y$-update can be recast as a weighted least-squares problem subject to the iteratively updated weight matrix $\bm W$, lending the algorithm its name. To ensure the matrix defining $\bm W$ is invertible, and to improve numerical stability, we can also introduce a smoothing parameter $\gamma_n>0$ to the $\bm W$-update as $\bm W_n = (\bm Y^T\bm Y+\gamma_n)^{-\frac{1}{2}}$, satisfying $\gamma_n\rightarrow 0$ as $n\rightarrow \infty$.

Making the substitution $\bm Y = \phi_d(\bm X)$, and replicating the steps above gives the following IRLS approach for solving \eqref{eq:Schattenp} with $p=1$:
\begin{align*}
\bm W_n & = (k(\bm X_n,\bm X_n)+\gamma_n I)^{-\frac{1}{2}}\\
\bm X_{n+1} & = \argmin_{\bm X}~\mathsf{tr}[k(\bm X,\bm X)\bm W_n]~~\mbox{s.t.}~~\mathcal{P}_\Omega(\bm X) = \mathcal{P}_\Omega(\bm X_0)
\end{align*}
Rather than finding the exact minimum in the $\bm X$ update, which could be costly, following the approach in \cite{mohan2012iterative}, we instead take a single projected gradient descent step to update $\bm X$. A straightforward calculation shows that the gradient of the objective $F(\bm X) = \mathsf{tr}[k(\bm X,\bm X)\bm W_n]$ is given by $\nabla F(\bm X) = \bm X(\bm W \odot k_{d-1}(\bm X,\bm X))$, where $\odot$ denotes an entry-wise product. Hence a projected gradient step is given by
\begin{align*}
\tilde{\bm X}_{n} & = \bm X_n -\tau \bm X_n(\bm W \odot k_{d-1}(\bm X_n,\bm X_n))\\
 \bm X_n  & =  \mathcal{P}_\Omega(\bm X_0) + \mathcal{P}_{\Omega^c}(\tilde{\bm X}) 
\end{align*}
where $\tau$ is a step-size parameter. 

\begin{Algorithm}[t]
\caption{Kernelized IRLS to solve \eqref{eq:Schattenp}.}
\label{alg:knn}
\begin{algorithmic}
\REQUIRE Samples $P_\Omega(\bm X_0)$, initialization of $\bm X$, polynomial kernel degree $d$, value $p$ of Schatten-$p$ penalty, $0 < p \leq 1$, IRLS parameters $\gamma,\eta$.
\STATE Set $q = 1-\frac{p}{2}$.
\WHILE{not converged}
  \STATE \textit{Step 1: Inverse power of kernel matrix}
  \STATE $\bm K \leftarrow k_d(\bm X, \bm X)$
  \STATE $(\bm V,\bm S) = \mathtt{eig}(\bm K)$.
  \STATE $\bm W \leftarrow \bm V (\bm S+\gamma \bm I)^{-q}\bm V^T$

  \vspace{1em}
  \STATE \textit{Step 2: Projected gradient descent step}
  \STATE $\tau \leftarrow \gamma^q$
  \STATE $\bm A \leftarrow \bm W\odot k_{d-1}(\bm X,\bm X)$
  \STATE $\bm X \leftarrow \bm X(\bm I-\tau\bm A)$
  \STATE $\bm X \leftarrow \mathcal{P}_{\Omega}(\bm X_0) + \mathcal{P}_{\Omega^c}(\bm X)$
  \vspace{1em}
  \STATE $\gamma \leftarrow \gamma / \eta$
\ENDWHILE
\end{algorithmic}
\end{Algorithm}

The above derivation can easily be extended to the Schatten-$p$ minimization problem \eqref{eq:Schattenp} by simply changing the $\bm W$-update to be the negative $q$th power of the kernel matrix, where $q = 1-\frac{p}{2}$. See Algorithm \ref{alg:knn} for pseudo-code of the proposed IRLS algorithm for solving \eqref{eq:Schattenp}. Heuristics are given in \cite{mohan2012iterative} for setting the optimization parameters $\gamma$ and $\tau$ based on $q$, which we adopt as well. Specifically, we set $\gamma_{n} = \gamma_0/\eta^n$, where $\gamma_0$ and $\eta$ are user-defined parameters, and update $\tau_{n} = \gamma_n^q$ where $q = 1-p/2$. The appropriate choice of $\gamma_0$ and $\eta$ will depend on the scaling and spectral properties of the data. Empirically, we find that setting $\gamma_0 = (0.1)^d \lambda_{max}$, where $\lambda_{max}$ is the largest eigenvalue of the kernel matrix obtained from the initialization, and $\eta = 1.01$ work well in a variety of settings. For all our experiments in Section \ref{sec:exp} we fix $p=1/2$, which was found to give the best matrix recovery results for synthetic data.

\begin{figure*}[ht]
\centering
\begin{tabular}{cc}
\includegraphics[width=0.60\textwidth]{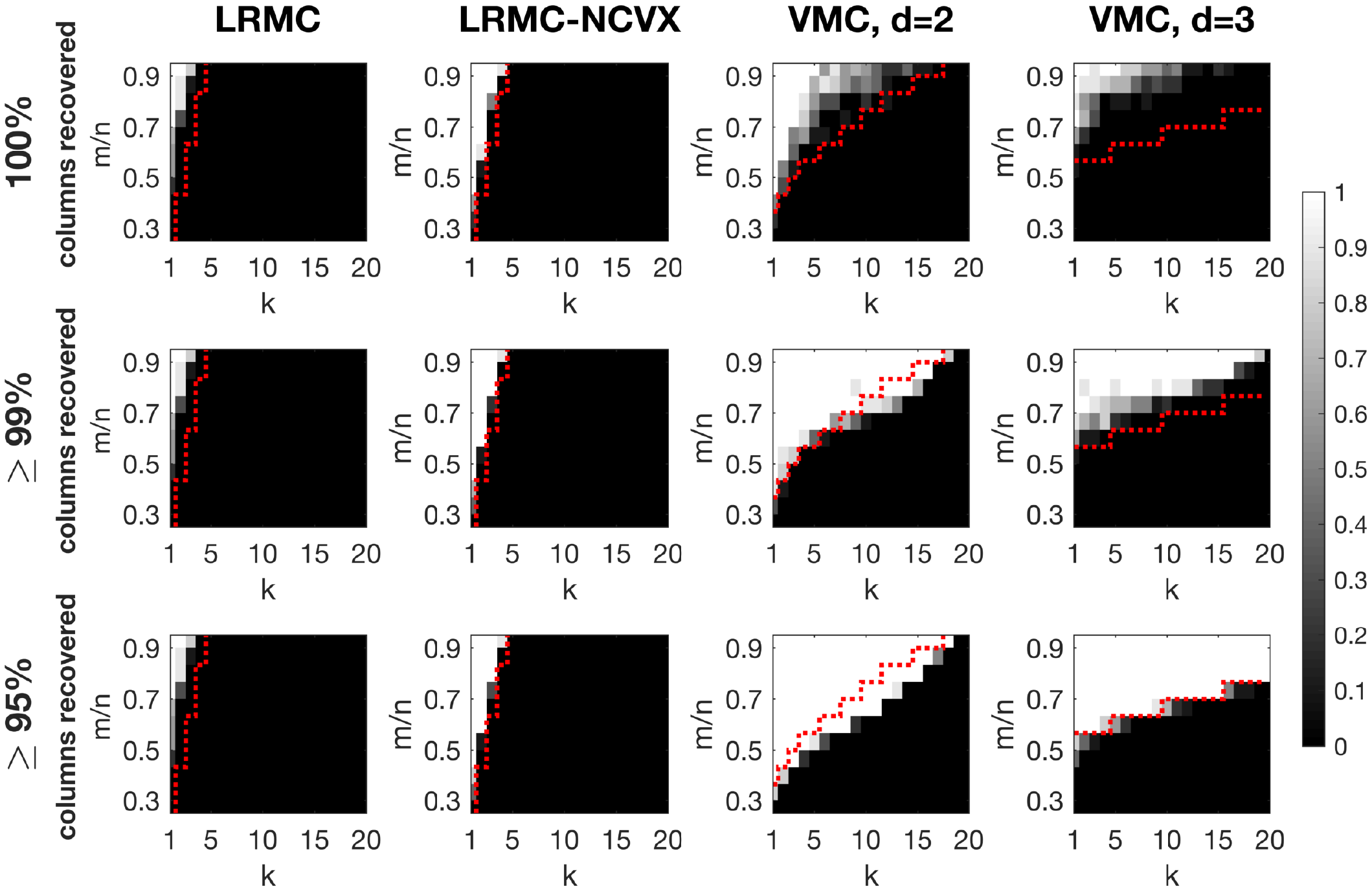} & 
\raisebox{1.0em}{\includegraphics[width=0.39\textwidth]{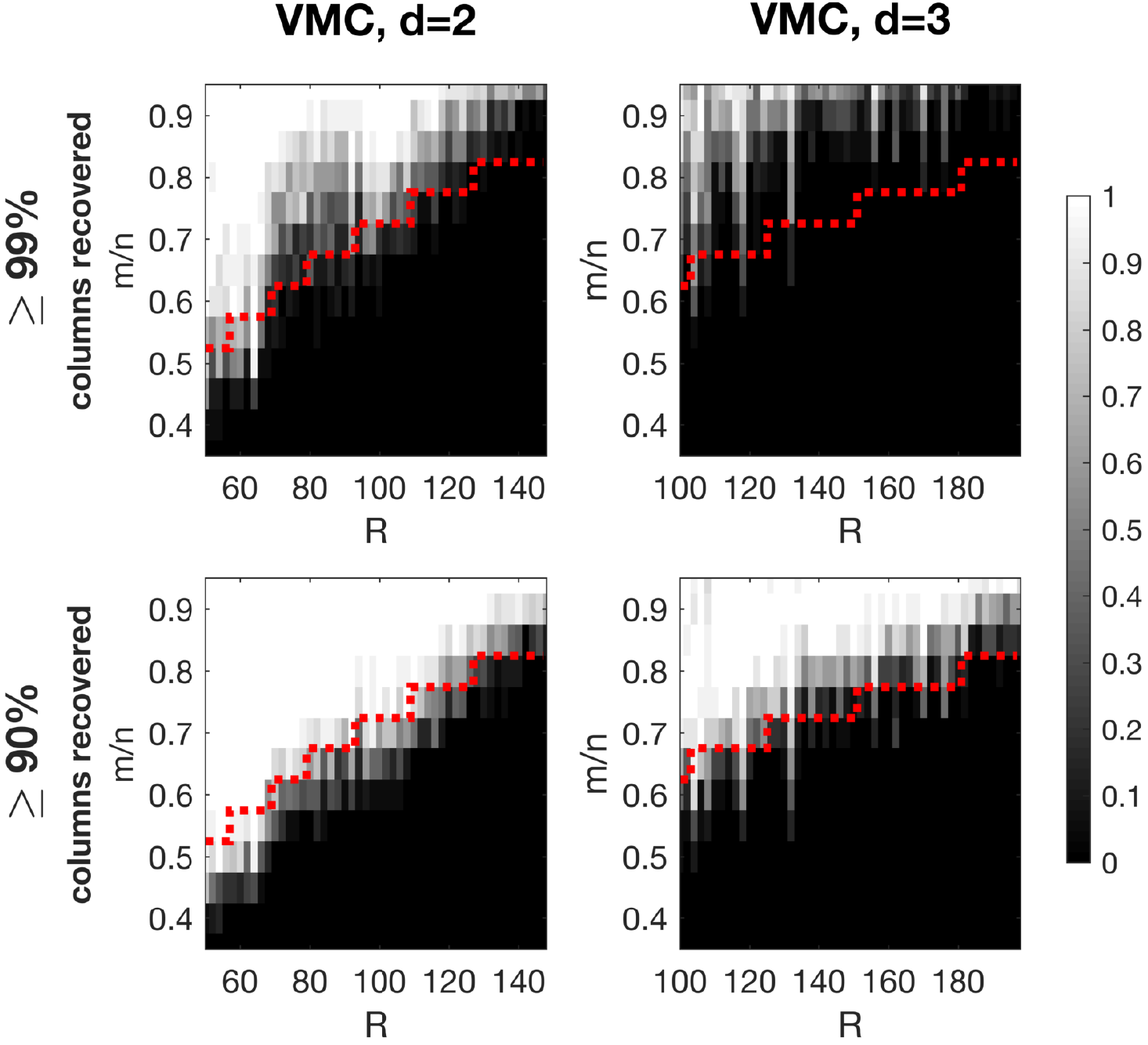}}
\\
(a) Union of Subspaces & (b) Parametric Data
\end{tabular}
\caption{Phase transitions for matrix completion of synthetic variety data. In (a) we simulate data belonging to a union of $k$ for varying $k$ (ambient dimension $n=15$, subspace dimension $r=3$, and $100$ data points sampled from each subspace giving $s=100 k$). In (b) we simulate data belonging union of few parametric curves and surfaces having known feature space rank $R$ (ambient dimension $n=20$ and $s=300$ data points). In all cases we undersample each column of the data matrix at a rate $m/n$, and perform matrix completion using the following algorithms: convex and non-convex low-rank matrix completion (LRMC,LRMC-NCVX) and our proposed variety-based matrix completion (VMC) approach for degree $d=2,3$. The grayscale values 0--1 indicate the fraction of random trials where the columns of the data matrix were successfully recovered up to the specified percentage. In all figures the red dashed line indicates the predicted minimal sampling rate $\rho_0 = m_0/n$ determined by \eqref{eq:minsamp}.}
\label{fig:phase}
\end{figure*}

\section{Numerical Experiments}\label{sec:exp}

\subsection{Empirical validation of sampling bounds}
In Figure \ref{fig:phase} we report the results of two experiments to validate the predicted minimum sampling rate $\rho_0$ in \eqref{eq:minsamp} on synthetic variety data. In the first experiment we generated $n\times s$ data matrices whose columns belong to a union of $k$ subspaces each of dimension $r$. We undersampled each column of the matrix taking $m$ entries uniformly at random, and attempted to recover the missing entries using our proposed IRLS algorithm for VMC (Algorithm \ref{alg:knn} with $p=1/2$) for degrees $d=2$ and $d=3$. As a baseline, we also compared with low-rank matrix completion (LRMC) via nuclear norm minimization in the original matrix domain. We also compare with matrix completion via non-convex Schatten-$1/2$ minimization (LRMC-NCVX) in the original matrix domain, which is implemented using Algorithm \ref{alg:knn}, with a linear kernel, \ie we set $d=1$ in \eqref{eq:kernel}. We said a column of the matrix was successfully completed when the relative recovery error $\|\bm x-\bm x_0\|/\|\bm x_0\|$ was less than $10^{-5}$, where $\bm x$ is the recovered column and $\bm x_0$ is the original column. We used the settings $n=15$, $s= 100 k$, $r=3$, for varying measurements $m$ and number of subspaces $k$, and measured the fraction of successful completions over $10$ random trials for each pair $(m,k)$. 

In the second experiment we attempted to recover synthetic variety data with pre-determined feature space rank. To construct this data, we sampled $s=300$ data points from a union of randomly generated parametric curves surfaces of dimension $1$, $2$ or $3$ belonging to $\mathbb{R}^{n}, n=20$. We generated several examples and sorted each dataset by its empirically determined feature space rank $R$. We follow the same experimental procedure as before. Because the data was high-rank, recovery via LRMC and LRMC-NCVX failed in all instances, and we omit these results from Figure \ref{fig:phase}(b).

Consistent with our theory, we find that VMC is successful at recovering most of the data columns above the predicted minimum sampling rate and typically fails below it. While VMC often fails to recover $100\%$ of the columns near the predicted rate, in fact a large proportion of the columns $(\%99--\%90)$ are completed successfully right up to the predicted rate. Sometimes the recovery dips below the predicted rate (\eg VMC, $d=2$ in Fig.~\ref{fig:phase}(a) and VMC, $d=3$ in Fig.~\ref{fig:phase}(b)). However, since the predicted rate relies on what is likely an over-estimate of the true degrees of freedom, it is not surprising that the $VMC$ algorithm occasionally succeeds below this rate, too.
 
\begin{figure}[ht]
\centering
\begin{minipage}{0.75\linewidth}
\includegraphics[width=\linewidth]{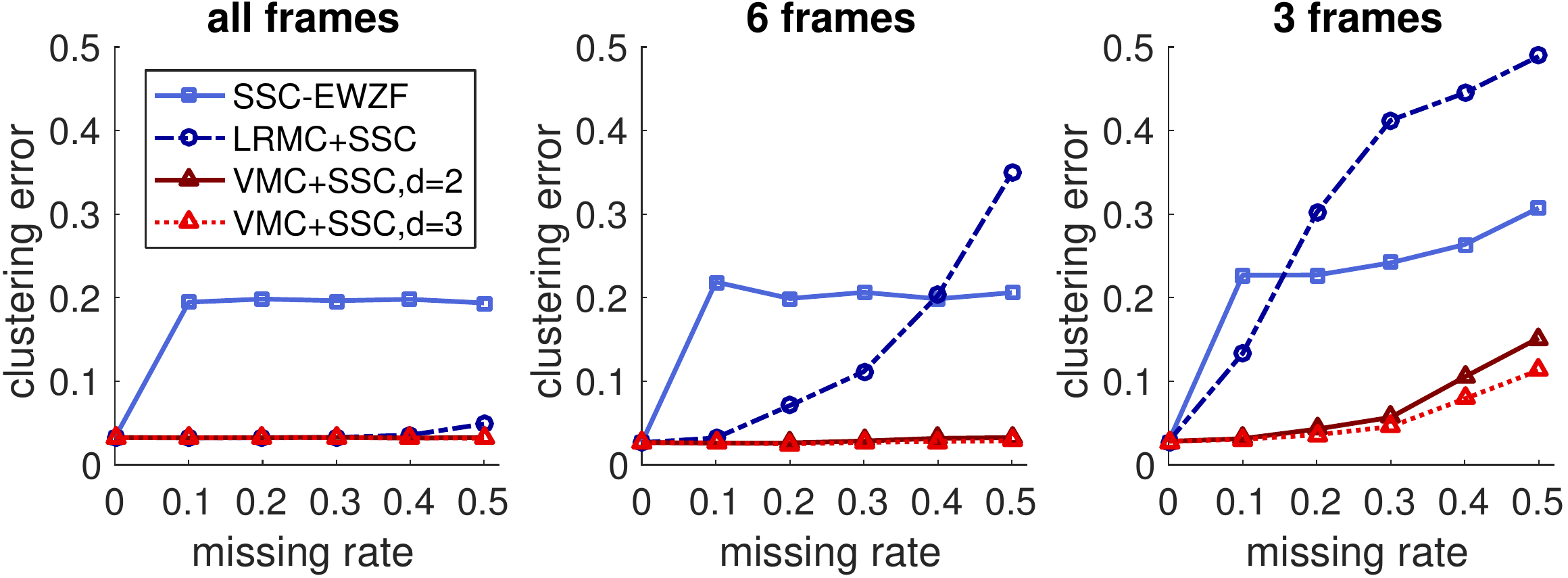}
\end{minipage}
\caption{Subspace clustering error on Hopkins 155 dataset for varying rates of missing data and undersampling of frames. The data is high-rank when fewer frames are sampled, in which case the proposed VMC+SSC approaches gives substantially lower clustering error over other state-of-the-art algorithms for subspace clustering with missing data.}
\label{fig:hopkins}
\end{figure}

\subsection{Motion segmentation of real data}
We also apply the VMC approach to the motion segmentation problem in computer vision \cite{kanatani2001motion} using real data belonging to the Hopkins 155 dataset \cite{tron2007benchmark}. This data consists of several feature points tracked across $F$ frames of the video, where the trajectories associated with one moving object can be modeled as lying in an low-dimensional affine subspace of $\mathbb{R}^{2F}$. We reproduce the experimental setting in \cite{yang2015sparse}, and simulate the high-rank setting by undersampling frames of the data. We also simulate missing trajectories by sampling uniformly at random from the feature points across all frames.

To obtain a clustering from missing data we first completed the missing entries using VMC and then ran the sparse subspace clustering (SSC) algorithm \cite{elhamifar2009sparse} on the result; we call the combination of these methods VMC+SSC. A similar approach of standard LRMC followed by SSC (LRMC+SSC) has been shown to provide a consistent baseline for subspace clustering with missing data \cite{yang2015sparse,elhamifar2016high}, and we use this a basis for comparison. We used publicly available code to implement SSC\footnote{\url{http://www.vision.jhu.edu/ssc.htm}}. The SSC algorithm requires a regularization parameter $\lambda$. We use the data adaptive choice $\lambda = \alpha/\min_j \max_{i\neq j} |\bm X^T\bm X|_{i,j}$ proposed in \cite{elhamifar2013sparse}, with $\alpha = 100$ for all experiments. We also compare against SSC with entry-wise zerofill (SSC-EWZF), which was the most competitive method among those proposed in \cite{yang2015sparse} for the Hopkins 155 dataset with missing data.

In Figure \ref{fig:hopkins} we show the results of our subspace clustering experiment on the Hopkins 155 dataset. We report the \emph{clustering error} of each algorithm, i.e., the proportion of data points clustered incorrectly, in the case where we retain all frames, 6 frames, and 3 frames of the dataset, and over a range of missing data rates. We find the VMC+SSC approach gives similar or lower clustering error than LRMC+SCC for low missing rates. Likewise, VMC+SSC also substantially outperforms SSC-EWZF for high missing rates. Note that unlike SSC-EWZF and the other algorithms introduced in \cite{yang2015sparse}, VMC+SSC also succeeds in setting where the data is low-rank (i.e., when all frames are retained). This is because the performance of VMC is similar to standard LRMC in the low-rank setting. 

\subsection{Completion of motion capture data}
Finally, we also consider the problem of completing time-series trajectories from motion capture sensors. We experiment on a dataset drawn from the CMU Mocap database\footnote{\url{http://mocap.cs.cmu.edu}}. Empirically, this dataset has been shown to be locally low-rank over the time frames corresponding to each separate activity, and can be modeled as a union of subspaces \cite{elhamifar2016high}. We chose a dataset (subject $56$, trial $6$) consisting of nine distinct motions in succession: punch, grab, skip, \emph{etc.}. The data had measurements from $n=62$ sensors at $s=6784$ time instants. We randomly undersampled the columns of this matrix and attempt to complete the data using VMC, LRMC, and LRMC-NCVX and measure the resulting \emph{completion error}: $\|\bm X-\bm X_0\|_F/\|\bm X_0\|_F$, where $\bm X$ is the recovered matrix and $\bm X_0$ is the original matrix. We show the results of this experiment in Figure \ref{fig:exp_mocap}. Similar to results on synthetic data, we find the VMC approach outperforms LRMC and LRMC-NCVX for appropriately chosen degree $d$. In particular, VMC with $d=2,3$ perform similar for small missing rates, but VMC $d=2$ gives lower completion error over $d=3$ for large missing rates, consistent with the results in Figure \ref{fig:phase}.

\begin{figure}[ht]
\centering
\begin{minipage}{0.5\linewidth}
\includegraphics[width=\linewidth]{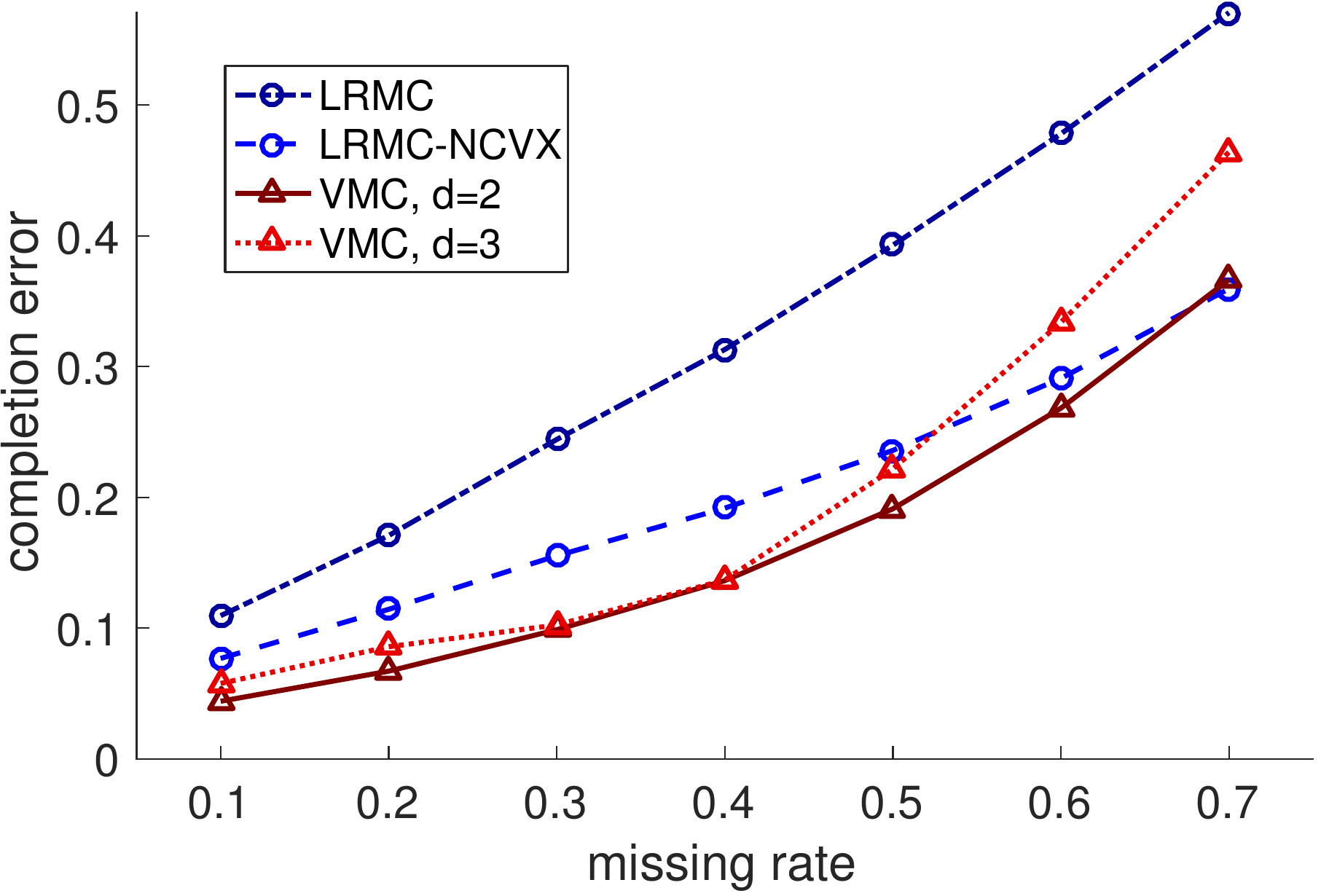}
\end{minipage}
\caption{Completion error on CMU Mocap dataset using the proposed VMC approach compared with convex and non-convex LRMC algorithms.}
\label{fig:exp_mocap}
\end{figure}

\section{Discussion}
An interesting feature of the IRLS algorithm presented in this work is that it can easily accommodate other smooth kernels, including the popular Gaussian radial basis function (RBF) \cite{muller2001introduction}. A similar optimization formulation to ours was presented in the recent pre-print \cite{garg2016non} using Gaussian RBF kernels in place of polynomial kernels, showing good empirical results in a matrix completion context. However, the analysis of the sampling complexity in this case is complicated by the fact that a feature space representation for Gaussian RBF kernel is necessarily infinite-dimensional. Understanding the sampling requirements in this case would be an interesting avenue for future work.

\section{Conclusion}
We introduce a matrix completion approach based on modeling data as an algebraic variety that generalizes low-rank matrix completion to a much wider class of data models, including data belonging to a union of subspaces. We present a hypothesized sampling complexity bound for the completion of a matrix whose columns belong to an algebraic variety. An surprising result of our analysis that that a union of $k$ affine subspaces of dimension $r$ should be recoverable from $O(r k^{1/d})$ measurements per column, provided we have $O(r^d)$ data points (columns) per subspace, where $d$ is the degree of the feature space map. \becca{In particular, if we choose $d = \log k$, then we need only $O(r)$ measurements per column as long as we have $O(r^{\log k})$ columns per subspace.} 
We additionally introduce an efficient algorithm based on an iterative reweighted least squares approach that realizes these hypothesized bounds on synthetic data, and reaches state-of-the-art performance on for matrix completion on several real high-rank datasets.

\bibliography{refs}
\bibliographystyle{ieeetr}

\end{document}